\documentclass[journal]{IEEEtran}

\usepackage{amsmath,amssymb,bm,bbm,mathrsfs,amscd}
\usepackage{color}
\usepackage{amsthm}
\usepackage{dsfont}
\usepackage{graphicx}
\usepackage{epsfig}
\usepackage{subcaption}
\usepackage{tikz}
\usepackage{balance}
\usepackage[linesnumbered,ruled]{algorithm2e}
    \SetKwProg{Init}{Initialize}{}{}
\usepackage{mwe}
\usepackage{float}

\newtheorem{theorem}{Theorem}
\newtheorem{definition}{Definition}

\newtheorem{lemma}{Lemma}

\newtheorem{remark}{Remark}
\newtheorem{assumption}{Assumption}

\newcommand{\RR}{{\mathbb{R}}}
\newcommand{\NN}{{\mathbb{N}}}
\newcommand{\EE}{{\mathbb{E}}}
\newcommand{\PP}{{\mathbb{P}}}
\newcommand{\JJ}{{\mathbb{J}}}
\newcommand{\FF}{{\mathbb{F}}}

\newcommand{\mc}{\mathcal}
\newcommand{\norm}[1]{\|#1\|}
\newcommand{\normsq}[1]{\|#1\|^2}
\newcommand{\EEk}[1]{\EE\left[#1|\mc F_k\right]}

\newcommand{\op}{\operatorname}

\newcommand{\bs}{\boldsymbol}
\newcommand{\fineass}{\hfill\small$\blacksquare$}

\ifCLASSINFOpdf
\else
\fi

\begin{document}
\bstctlcite{IEEEexample:BSTcontrol}

\title{Training Generative Adversarial Networks via stochastic Nash games}

\author{Barbara Franci$^{1}$ and Sergio Grammatico$^{1}$% <-this % stops a space
%\thanks{*This work was not supported by any organization}% <-this % stops a space
\thanks{$^{1}$The authors are with the Delft Center for System and Control, TU Delft, The Netherlands
        {\tt\footnotesize \{b.franci-1, s.grammatico\}@tudelft.nl}}%
\thanks{This work was partially supported by NWO under research projects OMEGA (613.001.702) and P2P-TALES (647.003.003), and by the ERC under research project COSMOS (802348).}}

\markboth{Journal of \LaTeX\ Class Files,~Vol.~14, No.~8, August~2015}%
{Shell \MakeLowercase{\textit{et al.}}: Bare Demo of IEEEtran.cls for IEEE Journals}

\maketitle

\begin{abstract}
Generative adversarial networks (GANs) are a class of generative models with two antagonistic neural networks: a generator and a discriminator. These two neural networks compete against each other through an adversarial process that can be modeled as a stochastic Nash equilibrium problem. Since the associated training process is challenging, it is fundamental to design reliable algorithms to compute an equilibrium. 

%For instance, the antagonistic behaviour can be modelled with a stochastic Nash equilibrium problem. Then, exploiting the connection with variational inequalities it is possible to design an algorithm for seeking a solution.
%by exploiting the fact that it can be rewritten as a stochastic Nash equilibrium problem. Then, thanks to the connection with variational inequalities it is possible to design an algorithm for seeking a solution.

In this paper, we propose a stochastic relaxed forward-backward (SRFB) algorithm for GANs and we show convergence to an exact solution 
when an increasing number of data is available. We also show convergence of an averaged variant of the SRFB algorithm to a neighborhood of the solution when only few samples are available. In both cases, convergence is guaranteed when the pseudogradient mapping of the game is monotone. This assumption is among the weakest known in the literature. Moreover,
we apply our algorithm to the image generation problem. 
\end{abstract}

\begin{IEEEkeywords}
Generative Adversarial Networks, stochastic Nash equilibrium problems, variational inequalities, two-player game.
\end{IEEEkeywords}

\section{Introduction}
\subsection{Generative Adversarial Networks}
Generative adversarial networks (GANs) are an example of unsupervised generative model. The basic idea is that, given some samples drawn from a probability distribution, the neural network takes a training set and learns how to obtain an estimate of such distribution. Most of the literature on GANs focuses on sample generation (especially image generation), but they can also be designed to explicitly estimate a probability distribution \cite{goodfellow2016,goodfellow2014,wang2019}.

The learning process of the neural networks in GANs is made via an adversarial process, in which not only the generative model, but also the opponent, are simultaneously trained. Indeed, there are two neural network classes: the generator that creates data according to a given distribution, and the discriminator that tries to recognize if the samples come from the training data or from the generator. 
As an example, the generator can be considered as a team of counterfeiters, trying to produce fake currency, while the discriminative model, i.e., the police, tries to detect the counterfeit money \cite{goodfellow2014}. To succeed in this game, the former must learn to reproduce money that are indistinguishable from the original currency, while the discriminator must recognize the samples that are drawn from the same distribution as the training data. Through the competition, both teams improve their methods until the counterfeit currency is indistinguishable from the original.

Besides this simplistic interpretation, the subject has been widely studied in the literature, because it has many and various applications. In addition to the classic image generation problem \cite{goodfellow2016,song2020}, GANs have been applied in medicine, e.g., to improve the diagnostic performance of the low-dose computed tomography method \cite{yang2018} and recently to detect pneumonia in potential Covid-19 patients \cite{khalifa2020}. Moreover, they can be used for correcting images taken under adverse weather conditions (as rain) \cite{zhang2019}, \cite{xiang2019}, editing facial attributes \cite{zhang2020}, image inpainting \cite{chen2020research,chen2020improved} as well as Pacman \cite{kim2020}.

\subsection{Stochastic Nash equilibrium problems}
The reason why these networks are called \textit{adversarial} is related to the fact that they can be modeled as a game, where each agent payoff depends on the variables of the other agent \cite{rotabulo2016,oliehoek2017}. However, the players in GANs can be also considered as cooperative players since they share information with each other \cite{goodfellow2016,oliehoek2017,heusel2017}.
Since there are only the generator and the discriminator, the problem is an instance of a two-player game and it can be also casted as a zero-sum game, depending on choice of the cost functions. From a more general point of view, the class of games that suits the GAN problem is that of stochastic Nash equilibrium problems (SNEPs) where each agent tries to minimize its expected value cost function. 
Given their connection with game theory, GANs have received theoretical attention as well, both on the study of the associated Nash equilibrium problem \cite{oliehoek2017,mazumdar2020} and on the design of algorithms to improve the learning and training process \cite{mazumdar2020,gidel2019}.

Among the available methods to solve a SNEP, an elegant approach is to recast the problem as a stochastic variational inequality (SVI) \cite{gidel2019,facchinei2007,tao2014}.
The advantage of this approach is that there are many algorithms available for finding a solution of an SVI, some of them already applied to GANs \cite{gidel2019,mertikopoulos2018}. For instance, the most used in machine learning is the forward-backward algorithm \cite{robbins1951}, also known as gradient descent \cite{bruck1977}, which has the disadvantage that, to ensure convergence, the mapping should be cocoercive, i.e., strongly monotone and Lipschitz continuous. Since the GAN mapping is often non-convex \cite{heusel2017,gidel2019}, one would prefer an algorithm that is guaranteed to converge for at most monotone mappings. In this case, it is possible to consider the extragradient (EG) algorithm \cite{iusem2017,yousefian2014,mishchenko2020} and the forward-backward-forward (FBF) algorithm \cite{staudigl2019}. The main downside of these two methods is that they require two costly evaluations of the pseudogradient mapping, which is computationally expensive. Due to the large-scale problem size, the ideal algorithm should not be computationally demanding and it should be guaranteed to converge under non-restrictive assumption on the pseudogradient mapping.

\subsection{Contribution}
Motivated by the need for computationally light algorithms converging under weak assumptions, we propose an algorithm that requests only one computation of the pseudogradient mapping at each iteration and we show its convergence under mere monotonicity. Specifically, our contributions are summarized next.
\begin{itemize}
\item We propose a stochastic relaxed forward-backward (SRFB) algorithm and a variant with averaging (aSRFB) for the training process of GANs. The SRFB involves only one evaluation of the pseudogradient mapping at each iteration, therefore it is computationally cheaper than the EG and FBF algorithms. 
\item We prove its convergence for monotone mappings, which is considered the ``weakest possible" assumption on the pseudogradient mapping \cite{facchinei2010}. Specifically, whenever only a finite number of samples are available, we prove almost sure convergence to a neighborhood of the solution, while if an increasing set of samples is available, then the algorithm reaches an equilibrium almost surely.
\item We apply our algorithm to the image generation problem and compare it with the extragradient scheme.
\end{itemize}

Our SRFB algorithm is inspired by \cite{malitsky2019,franci2020} and a preliminary heuristic application to GAN was presented in \cite{franci2020cdc}. Therein, we do not prove convergence of the SRFB algorithm nor of its aSRFB variant. Moreover, in \cite{franci2020cdc}, we only run numerical simulations on synthetic toy examples while in this paper we train the two neural networks for the popular image generation problem with real benchmark data.

\subsection{Related work}
Due to the connection between SNEPs and SVIs, many algorithms for variational inequalities have been applied to GANs \cite{gidel2019,tao2014}.

The first one is the forward-backward (FB) algorithm \cite{franci2020tac,robbins1951}, also known as gradient descent \cite{bruck1977}. It is the most used, even if in many cases it has been proven to be non-convergent \cite{gidel2019,mescheder2018}. From an operator-theoretic perspective, the FB is not convergent because the pseudogradient mapping should be cocoercive \cite{grammatico2018} and this is almost never the case in GANs. 
From an algorithmic perspective, the iterates typically cycle in a neighbourhood of a solution without reaching it \cite{mescheder2018}.

Therefore, research has focused on the forward-backward-forward (FBF) algorithm and on the extragradient (EG) algorithm, that are guaranteed to converge for merely monotone mappings.
The FBF algorithm, first presented in \cite{tseng2000} and extended to the stochastic case in \cite{staudigl2019}, involves two evaluations of the pseudogradient mapping. A first attempt to apply the FBF algorithm for GANs is presented, along with a relaxed inertial FBF algorithm, in \cite{vuong2020}.
The extragradient (EG) method was first proposed in \cite{korpelevich1976} and extended many years later to the stochastic case in \cite{iusem2017,yousefian2014} and to GANs in \cite{gidel2019,mishchenko2020}.
The EG algorithm requires two evaluations of the pseudogradient mapping as well, therefore in \cite{gidel2019}, a variation is proposed. This involves an extrapolation from the past, i.e., it uses the evaluation of the mapping at previous time steps. In \cite{gidel2019} the authors propose also the FB and the EG algorithms with averaging.

The averaging technique was first proposed for VIs in \cite{bruck1977} and studied more recently in \cite{gidel2019,yazici2018}.
In \cite{yazici2018}, the authors examine two different techniques for averaging: the moving average, which computes the time-average of the iterates, and the exponential moving average which computes an exponentially discounted sum.
For both the techniques, they show that, despite convergence cannot be proven, the averaging may help stabilizing the iterates, driving them towards a neighborhood of the solution.

While \cite{yazici2018} has mostly a heuristic approach, theoretical convergence studies are presented in \cite{mescheder2017, mescheder2018}. Therein, the authors show that local convergence and stability properties of GAN training depends on the eigenvalues of the Jacobian of the associated gradient vector field. 
 
Another theoretical aspect that has not been extensively addressed yet is the inherent relation between GANs and game theory. In \cite{oliehoek2017}, the authors formally introduce Generative Adversarial Network Games, describing (and seeking for) the Nash equilibria of the zero-sum game as saddle points in mixed strategies.
The study of saddle-point problems is also studied, in connection with GANs, in \cite{mertikopoulos2018}. The authors in \cite{heusel2017}, instead, prove that Adam \cite{kingma2014}, a second order method for GANs, converges to a stationary local Nash equilibrium.

\subsection{Notation} 
Let $\RR$ indicate the set of real numbers and let $\bar\RR=\RR\cup\{+\infty\}$.
$\langle\cdot,\cdot\rangle:\RR^n\times\RR^n\to\RR$ denotes the standard inner product and $\norm{\cdot}$ represents the associated Euclidean norm. Given $N$ vectors $x_{1}, \ldots, x_{N} \in \RR^{n}$, $\boldsymbol{x} :=\op{col}\left(x_{1}, \dots, x_{N}\right)=\left[x_{1}^{\top}, \dots, x_{N}^{\top}\right]^{\top}.$
For a closed set $C \subseteq \RR^{n},$ the mapping $\op{proj}_{C} : \RR^{n} \to C$ denotes the projection onto $C$, i.e., $\op{proj}_{C}(x)=\op{argmin}_{y \in C}\|y-x\|$.

\section{Generative Adversarial Networks}

The idea behind generative adversarial networks (GANs) is to set up an antagonistic training process between the generator and the discriminator. Typically, the generator and the discriminator are represented by two deep neural networks, and accordingly, they are denoted by two functions, differentiable with respect to their inputs and parameters.
The generator creates samples that aim at resembling the distribution of the training data. The generator is therefore trained to fool the discriminator who, in turn, examines the samples to determine whether they are real or fake. 
This adversarial mechanism can be modeled as a game where the generator and the discriminator represent the players, who want to improve their \textit{payoff} \cite{oliehoek2017}.

Formally, the generator is a neural network class, represented by a differentiable function $g$, with parameters vector $x_{\text{g}} \in \Omega_{\text{g}}\subseteq\RR^{n_{\text{g}}}$. Let us denote the (fake) output of the generator with $g(z,x_{\text{g}}) \in \mathbb{R}^{q}$ where the input $z$ is a random noise vector drawn from the data prior distribution, $z \sim p_{\text{z}}$ \cite{oliehoek2017}.
In game-theoretic terms, the \textit{strategies} of the generator are the parameters $x_{\text{g}}$ that allow $g$ to generate the fake output.

Similarly to the generator, the discriminator is a neural network class with parameter vector $x_{\text{d}} \in \Omega_{\text{d}}\subseteq\RR^{n_{\text{d}}}$ and a single output $d(v,x_{\text{d}}) \in[0,1]$ that indicates how good is the input $v$. The output of the discriminator can be interpreted as the probability of being real that $d$ assigns to an element $v$. The \textit{strategies} of the discriminator are the parameters $x_{\text{d}}$.

Usually \cite{gidel2019,goodfellow2014}, the payoff of the discriminator is given by the function
\begin{equation}\label{eq_payoff_zero}
J_{\text{d}}(x_{\text{g}},x_{\text{d}})=\EE[\phi(d(\cdot,x_{\text{d}})]-\EE[\phi(d(g(\cdot,x_{\text{g}}),x_{\text{d}}))],
\end{equation}
where $\phi:[0,1] \rightarrow \mathbb{R}$ is a measuring function. The typical choices for $\phi$ are the Kullback-Leibler divergence or the Jensen-Shannon divergence (a logarithm) as in \cite{goodfellow2014} but other options (such as the Wasserstein distance) are proposed in the literature \cite{wang2019,arjovsky2017}. Regardless, the mapping in \eqref{eq_payoff_zero} can be interpreted as the distance between the fake value and the real one. The payoff of the generator ($J_\text{g}$) instead, depends on how we describe the game. In fact, the problem can be modeled as a two-player game, or as a zero-sum game, depending on the cost functions. 
To cast the problem as a zero-sum game, the functions $J_\text{g}$ and $J_\text{d}$ should satisfy the following relation: 
\begin{equation}\label{eq_rel_zero}
J_{\text{g}}(x_{\text{g}},x_{\text{d}})=-J_{\text{d}}(x_{\text{g}},x_{\text{d}}).
\end{equation}
Then, we can rewrite it as a minmax problem, i.e.,
\begin{equation}\label{eq_minmax}
\min_{x_{\text{g}}}\max_{x_{\text{d}}} J_{\text{d}}(x_{\text{g}},x_{\text{d}}).
\end{equation}
In words, \eqref{eq_minmax} means that the generator aims at minimizing the distance between the real value and the fake one, while the discriminator wants to maximize such a distance, i.e., $d$ aims at recognizing the generated data.

When the generator has a different payoff function from the discriminator, e.g., given by \cite{gidel2019}
\begin{equation}\label{eq_cost_gen}
J_{\text{g}}(x_{\text{g}},x_{\text{d}})=\EE[\phi(d(g(\cdot,x_{\text{g}}),x_{\text{d}}))],
\end{equation}
then the problem is not a zero-sum game.

Since the two-player game with cost functions \eqref{eq_payoff_zero} and \eqref{eq_cost_gen} and the zero-sum game with cost function \eqref{eq_payoff_zero} and relation \eqref{eq_rel_zero} have the same pseudogradient mapping (defined Section \ref{sec_snep}), it can be proven that the two equilibria are strategically equivalent \cite[Th. 10]{oliehoek2017}.

\section{Stochastic Nash equilibrium problems}\label{sec_snep}

In this section, let us describe the GAN game as a generic stochastic Nash equilibrium problem (SNEP). 

The two neural network classes are indexed by the set $\mc I=\{\text{g},\text{d}\}$. Each agent $i\in\mc I$ has a decision variable $x_i\in\Omega_i\subseteq\RR^{n_i}$. In general, the local cost function of agent $i\in \mc I$ is defined as 
\begin{equation}\label{eq_cost_stoc}
\JJ_i(x_i,x_j)=\EE_\xi[J_i(x_i,x_j,\xi(\omega))],
\end{equation}
for some measurable function $J_i:\mc \RR^{n}\times \RR^d\to \RR $ where $n=n_{\text{d}}+n_{\text{g}}$. The cost function $\JJ_i $ of agent $i\in\mc I $ depends on its local variable $x_i$, the decisions of the other player $x_j$, $j\neq i$, and the random variable $\xi:\Xi\to\RR^d $ that represent the uncertainty. 
The latter arises when we do not know the distribution of the random variable or it is computationally too expansive to compute the expected value. In practice, this means that we have access only to a finite number of samples from the data distribution.
Given the probability space $(\Xi, \mc F, \PP) $, $\EE_\xi $ indicates the mathematical expectation with respect to the distribution of the random variable $\xi(\omega) $\footnote{From now on, simplicity, we use $\xi $ instead of $\xi(\omega) $ and $\EE $ instead of $\EE_\xi $.}. Let us suppose that $\EE[J_i(\bs{x},\xi)] $ is well defined for all the feasible $\bs{x}=\op{col}(x_{\text{g}},x_{\text{d}})\in\bs \Omega=\Omega_{\text{g}}\times\Omega_{\text{d}} $. 
For our theoretical analysis, some assumptions on the cost function and the feasible set should be postulated. The following assumptions are standard in monotone game theory \cite{facchinei2007vi,ravat2011}.
\begin{assumption}\label{ass_J}
For each $i \in \mc I,$ the set $\Omega_{i}$ is nonempty, compact and convex.\\
For each $i,j \in \mc I $, $i\neq j$, the function $\JJ_{i}(\cdot, x_j) $ is convex and continuously differentiable.
%\fineass\end{assumption}
%\begin{assumption}\label{ass_J_exp}
For each $i\in\mc I $, $j\neq i$ and for each $\xi \in \Xi $, the function $J_{i}(\cdot,x_j,\xi) $ is convex, continuously differentiable and Lipschitz continuous with the constant $\ell_i(x_j,\xi) $ integrable in $\xi $. The function $J_{i}(x_i,x_j,\cdot) $ is measurable for each $x_j $, $j\neq i$.
%\fineass\end{assumption}
%\begin{assumption}\label{ass_omega}
%For each $i \in \mc I,$ the set $\Omega_{i}$ is nonempty, compact and convex.
\fineass\end{assumption}

%\begin{algorithm}[t]
%\caption{Stochastic Relaxed forward-backward (SRFB)}\label{algo_i}
%Initialization: $x_i^0 \in \Omega_i$\\
%Iteration $k$: Agent $i$ receives $x_j^k$ for $j \neq i$, then updates:
%\begin{subequations}
%\begin{align}
%\bar{x}_i^{k} &=(1-\delta) x_i^k+\delta\bar{x}_i^{k-1} \label{eq_relax}\\ 
%x_i^{k+1}&=\op{proj}_{\Omega_i}[\bar{x}_i^k-\lambda_{i}F^{\textup{VR}}_{i}(x_i^k, x_j^k,\xi_i^k)]
%\end{align}
%\end{subequations}
%\end{algorithm}

Given the decision variable of the other agent, the aim of each agent $i$ is to choose a strategy $x_i$ that solves its local optimization problem, i.e.,
\begin{equation}\label{eq_game}
\forall i \in \mc I: \quad
\min\limits _{x_i \in \Omega_i}  \JJ_i\left(x_i, x_j\right).\\ 
\end{equation}
The solution of the coupled optimization problems in \eqref{eq_game} that we are seeking is a stochastic Nash equilibrium (SNE) \cite{ravat2011}.

\begin{definition}\label{def_GNE}
A stochastic Nash equilibrium is a collective strategy $\bs x^*=\op{col}(x_{\text{g}}^*,x_{\text{d}}^*)\in\bs{\Omega} $ such that for all $i \in \mc I $
 $$\JJ_i(x_i^{*}, x_j^{*}) \leq \inf \{\JJ_i(y, x_j^{*})\; | \; y \in \Omega_i\}.$$%\vspace{-.5cm} $$
% \fineass
\end{definition}
In other words, a SNE is a pair of strategies where neither the generator, nor the discriminator, can decrease its cost function by unilaterally deviating from its decision.

While existence of a SNE of the game in \eqref{eq_game} is guaranteed, under Assumption \ref{ass_J} \cite[Section 3.1]{ravat2011}, uniqueness does not hold in general \cite[Section 3.2]{ravat2011}.

To seek for a Nash equilibrium, we rewrite the problem as a stochastic variational inequality (SVI). 
Let us first denote the pseudogradient mapping as
\begin{equation}\label{eq_grad}
\FF(\bs{x})=\left[\begin{array}{c}
\EE[\nabla_{x_{\text{g}}} J_{\text{g}}(x_{\text{g}}, x_{\text{d}})]\\
\EE[\nabla_{x_{\text{d}}} J_{\text{d}}(x_{\text{d}}, x_{\text{g}})]
\end{array}\right].
\end{equation} 
We note that the possibility to exchange the expected value and the pseudogradient is ensured by Assumption \ref{ass_J} \cite{ravat2011}.

Then, the associated SVI reads as
\begin{equation}\label{eq_SVI}
\langle \FF(\bs x^*),\bs x-\bs x^*\rangle\geq 0\text { for all } \bs x \in \bs{\Omega}.
\end{equation}

\begin{remark}\label{remark_vsne}
If Assumption \ref{ass_J} holds, then $\bs x^*\in\bs{\Omega}$ is a Nash equilibrium of the game in \eqref{eq_game} if and only if $\bs x^*$ is a solution of the SVI in \eqref{eq_SVI} \cite[Prop. 1.4.2]{facchinei2007}, \cite[Lem. 3.3]{ravat2011}. 

Moreover, under Assumption \ref{ass_J}, the solution set of $\op{SVI}(\bs{\Omega},\FF) $ is non empty and compact, i.e., $\op{SOL}(\bs{\Omega},\FF)\neq\varnothing $ \cite[Corollary 2.2.5]{facchinei2007} and an equilibrium exists.
\fineass\end{remark}

%\begin{remark}\need{shall we put Remark 1 and 2 together?}
%Under Assumptions \ref{ass_J} and \ref{ass_omega}, the solution set of $\op{SVI}(\bs{\Omega},\FF) $ is non empty and compact, i.e., $\op{SOL}(\bs{\Omega},\FF)\neq\varnothing $ \cite[Corollary 2.2.5]{facchinei2007}, that is, a v-SNE exists.
%\fineass\end{remark}

In light of Remark \ref{remark_vsne}, we call variational equilibria (v-SNE) the solution of the $\op{SVI}(\bs{\Omega} , \FF) $ in (\ref{eq_SVI}) where $\FF $ is as in (\ref{eq_grad}), i.e., the solution of the SVI that are also SNE.

\section{Stochastic relaxed forward-backward algorithms}
%\begin{algorithm}[t]
%\caption{Stochastic Relaxed forward-backward (SRFB)}\label{algo_i}
%Initialization: $x_i^0 \in \Omega_i$\\
%Iteration $k$: Agent $i$ receives $x_j^k$ for $j \neq i$, then updates:
%\begin{subequations}
%\begin{align}
%\bar{x}_i^{k} &=(1-\delta) x_i^k+\delta\bar{x}_i^{k-1} \label{eq_relax}\\ 
%x_i^{k+1}&=\op{proj}_{\Omega_i}[\bar{x}_i^k-\lambda_{i}F^{\textup{VR}}_{i}(x_i^k, x_j^k,\xi_i^k)]
%\end{align}
%\end{subequations}
%\end{algorithm}

%\begin{algorithm}[t]
%\caption{Stochastic Relaxed forward-backward with averaging (aSRFB)}\label{algo_i2}
%Initialization: $x_i^0 \in \Omega_i$\\
%Iteration $k\in\{1,\dots,K\}$: Agent $i$ receives $x_j^k$ for $j \neq i$, then updates:
%\begin{subequations}
%\begin{align}
%\bar{x}_i^{k} &=(1-\delta) x_i^k+\delta\bar{x}_i^{k-1} \label{eq_relax2}\\ 
%x_i^{k+1}&=\op{proj}_{\Omega_i}[\bar{x}_i^k-\lambda_{i}F^{\textup{SA}}_{i}(x_i^k, x_j^k,\xi_i^k)]
%\end{align}
%\end{subequations}
%Iteration $K$: $X^K_i=\frac{\sum\nolimits_{k=1}^K\lambda_kx_i^k}{\sum\nolimits_{k=1}^K\lambda_k}$
%\end{algorithm}
In this section, we propose two algorithms for solving the SNEP associated to the GANs process: a stochastic relaxed forward backward (SRFB) algorithm and its variant with averaging (aSRFB). The iterations read as in Algorithm \ref{algo_i} and Algorithm \ref{algo_i2}, respectively and they represent the steps for each agent $i\in\{\textup{g},\textup{d}\}$.

Algorithm \ref{algo_i} and Algorithm \ref{algo_i2} differ, besides the presence of the averaging step, on the choice of the approximation used for the pseudogradient mapping.
Moreover, we note that the averaging step in Algorithm \ref{algo_i2}, namely,
\begin{equation}\label{eq_ave}
\bs X^{K} = \frac{\sum_{k=1}^{K} \lambda_{k} \bs x^k}{S_{K}}, \quad S_{K} = \sum_{k=1}^{K} \lambda_{k}
\end{equation}
can be implemented in a first-order fashion as 
\begin{equation}\label{eq_ave_online}
\bs X_K=(1-\tilde{\lambda}_{K})\bs X^{K-1}+\tilde{\lambda}_{K} \bs x^{K}
\end{equation}
for some $ \tilde{\lambda}_{K} \in[0, 1]$. Moreover, let us remark that \eqref{eq_ave_online} is different from \eqref{eq_relax} and \eqref{eq_relax2}. Indeed, in Algorithms \ref{algo_i} and \ref{algo_i2}, \eqref{eq_relax} and \eqref{eq_relax2} are convex combinations, with a constant parameter $\delta$, of the two previous iterates $\bs x^k$ and $\bar{\bs x}^{k-1}$, while the averaging in \eqref{eq_ave_online} is a weighted cumulative sum over the decision variables $\bs x^k$ for all the iterations $k\in\{1,\dots,K\}$, with time--varying weights $\{\tilde\lambda_k\}_{k=1}^K$.
The parameter $\tilde{\lambda}_K$ can be tuned to obtain uniform, geometric or exponential averaging \cite{gidel2019,yazici2018}.

\begin{algorithm}[t]
\caption{Stochastic Relaxed forward-backward (SRFB)}\label{algo_i}
Initialization: $x_i^0 \in \Omega_i$\\
Iteration $k$: Agent $i$ receives $x_j^k$ for $j \neq i$, then updates:
\begin{subequations}
\begin{align}
\bar{x}_i^{k} &=(1-\delta) x_i^k+\delta\bar{x}_i^{k-1} \label{eq_relax}\\ 
x_i^{k+1}&=\op{proj}_{\Omega_i}[\bar{x}_i^k-\lambda_{i}F^{\textup{VR}}_{i}(x_i^k, x_j^k,\xi_i^k)]
\end{align}
\end{subequations}
\end{algorithm}

Let us now describe the approximation schemes used in the definitions of the algorithms.
In the SVI framework, there are two main possibilities, depending on the samples available.

Using a finite, fixed number of samples is called stochastic approximation (SA) \cite{robbins1951} and it is widely used in the literature of SVIs, in conjunction with conditions on the step sizes to control the stochastic error \cite{yousefian2014,koshal2013}. In fact, unless the step size sequence is diminishing, it is only possible to prove convergence to a neighborhood of a solution. 
The stochastic approximation of the pseudogradient mapping, given one sample of the random variable reads as
\begin{equation}\label{eq_SA}
F^{\textup{SA}}(\bs x,\xi):=\left[\begin{array}{c}
\nabla_{x_{\text{g}}} J_{\text{g}}(x_{\text{g}}, x_{\text{d}},\xi_\text{g})\\
\nabla_{x_{\text{d}}} J_{\text{d}}(x_{\text{d}}, x_{\text{g}},\xi_\text{d})
\end{array}\right].
\end{equation}
$F^{\textup{SA}}$ uses one or a finite number, called mini-batch, of realizations of the random variable. 

When a huge number of samples is available, one can consider using a different approximation scheme, i.e.,
\begin{equation}\label{eq_SAA}
F^{\textup{VR}}(\bs x,\xi^k)=\left[\begin{array}{c}
\frac{1}{N_k}\sum_{s=1}^{N_k}\nabla_{x_{\text{g}}} J_i(x_{\text{g}}^k,x_{\text{d}}^k,\xi_{\text{g}}^{(s)})\\
\frac{1}{N_k}\sum_{s=1}^{N_k}\nabla_{x_{\text{d}}} J_i(x_{\text{d}}^k,x_{\text{g}}^k,\xi_{\text{d}}^{(s)})
\end{array}\right].
\end{equation}
In this case, an increasing number of samples, the batch size $N_k$, is taken at each iteration \cite{iusem2017}. The superscript VR stands for variance reduction and it is related to the property of the approximation error discussed in Remark \ref{remark_error}.

\begin{algorithm}[t]
\caption{Stochastic Relaxed forward-backward with averaging (aSRFB)}\label{algo_i2}
Initialization: $x_i^0 \in \Omega_i$\\
Iteration $k\in\{1,\dots,K\}$: Agent $i$ receives $x_j^k$ for $j \neq i$, then updates:
\begin{subequations}
\begin{align}
\bar{x}_i^{k} &=(1-\delta) x_i^k+\delta\bar{x}_i^{k-1} \label{eq_relax2}\\ 
x_i^{k+1}&=\op{proj}_{\Omega_i}[\bar{x}_i^k-\lambda_{i}F^{\textup{SA}}_{i}(x_i^k, x_j^k,\xi_i^k)]
\end{align}
\end{subequations}\\
Iteration $K$: $X^K_i=\frac{\sum\nolimits_{k=1}^K\lambda_kx_i^k}{\sum\nolimits_{k=1}^K\lambda_k}$
\end{algorithm}

\section{Convergence analysis}

\subsection{Basic technical assumptions}
With the aim of proving convergence to a solution (or to a neighborhood of one) of Algorithms \ref{algo_i} and \ref{algo_i2},
we start this section with some assumptions that are common to both the algorithms.

The following monotonicity assumption on the pseudogradient mapping is standard for SVI problems \cite{iusem2017,franci2020}, also when applied to GANs \cite{gidel2019} and it is the weakest possible to hope for global convergence.
\begin{assumption}\label{ass_mono}
$\FF$ in \eqref{eq_grad} is monotone, i.e., $\langle\FF(\bs x)-\FF(\bs y),\bs x-\bs y\rangle\geq0$ for all $\bs x,\bs y\in\bs \Omega$.
\fineass\end{assumption}

Let us now define the filtration $\mc F=\{\mc F_k\}$, that is, a family of $\sigma$-algebras such that $\mathcal{F}_{0} = \sigma\left(X_{0}\right)$,
$\mathcal{F}_{k} = \sigma\left(X_{0}, \xi_{1}, \xi_{2}, \ldots, \xi_{k}\right)$ for all $k \geq 1,$
and $\mc F_k\subseteq\mc F_{k+1}$ for all $k\geq0$. 
For all $k \geq 0,$ let us also define the stochastic error as
\begin{equation}\label{eq_error}
\epsilon_k= \hat F(\bs x^k,\xi^k)-\FF(\bs x^k),
\end{equation}
where $\hat F$ indicates one of the two possible approximation schemes.
In words, $\epsilon_k$ in \eqref{eq_error} is the distance between the approximation and the exact expected value mapping. Then, let us postulate that the stochastic error has zero mean and bounded variance, as usual in SVI \cite{gidel2019,iusem2017,franci2020}.

\begin{assumption}\label{ass_error}
The stochastic error is such that, for all $k\geq 0$, a.s.,
$\EE[\epsilon^k|\mc F_k]=0 \text{ and }\EE[\|\epsilon^k\|^2|\mc F_k]\leq\sigma^2.$
\fineass\end{assumption}

%\begin{algorithm}[t]
%\caption{Stochastic Relaxed forward-backward with averaging (aSRFB)}\label{algo_i2}
%Initialization: $x_i^0 \in \Omega_i$\\
%Iteration $k\in\{1,\dots,K\}$: Agent $i$ receives $x_j^k$ for $j \neq i$, then updates:
%\begin{subequations}
%\begin{align}
%\bar{x}_i^{k} &=(1-\delta) x_i^k+\delta\bar{x}_i^{k-1} \label{eq_relax2}\\ 
%x_i^{k+1}&=\op{proj}_{\Omega_i}[\bar{x}_i^k-\lambda_{i}F^{\textup{SA}}_{i}(x_i^k, x_j^k,\xi_i^k)]
%\end{align}
%\end{subequations}\\
%Iteration $K$: $X^K_i=\frac{\sum\nolimits_{k=1}^K\lambda_kx_i^k}{\sum\nolimits_{k=1}^K\lambda_k}$
%\end{algorithm}

\subsection{Convergence of Algorithm \ref{algo_i}}

We now state the convergence result for Algorithm \ref{algo_i}. First, let us postulate some assumptions functional to our analysis. We start with the batch size sequence, which should be increasing to control the stochastic error.

\begin{assumption}\label{ass_batch}
The batch size sequence $(N_k)_{k\geq 1}$ is such that, for some $b,k_0,a>0$,
$N_k\geq b(k+k_0)^{a+1}.$
\fineass\end{assumption}
\begin{remark}\label{remark_error}
Given $F^{\textup{VR}}$ as in \eqref{eq_SAA}, it can be proven that, for some $C>0$,
$$\textstyle{\EEk{\normsq{\epsilon_k}}\leq\frac{C\sigma^2}{N_k},}$$
i.e., the error diminishes as the batch size increases. Such result is, therefore, called variance reduction. More details can be found in \cite[Lemma 3.12]{iusem2017} \cite[Lemma 6]{franci2020tac}.
\fineass\end{remark}

In addition to Assumption \ref{ass_mono}, we postulate that the pseudogradient mapping is Lipschitz continuous.
\begin{assumption}\label{ass_lip}
$\FF$ as in \eqref{eq_grad} is $\ell$-Lipschitz continuous for $\ell>0$, i.e., $\norm{\FF(\bs x)-\FF(\bs y)}\leq\ell\norm{\bs x-\bs y}$ for all $\bs x,\bs y\in \bs\Omega$.
\fineass\end{assumption}

Using the variance reduced scheme in \eqref{eq_SAA}, we can take a constant step size, as long as it is small enough while the relaxation parameter should not be too small.
\begin{assumption}\label{ass_delta}
The step size in Algorithm \ref{algo_i} is such that $\lambda\in(0, \frac{1}{2\delta(2\ell+1)}]$
where $\ell$ is the Lipschitz constant of $\FF$ in \eqref{eq_grad} as in Assumption \ref{ass_lip}.
%\fineass\end{assumption}
%On the other hand, the relaxation parameter should not be too small.
%\begin{assumption}\label{ass_delta}
The relaxation parameter in Algorithm \ref{algo_i} is such that $\delta\in[\frac{2}{1+\sqrt{5}},1]$. 
\fineass\end{assumption}

We can finally state our first convergence result.
\begin{theorem}\label{theo_SAA}
Let Assumptions \ref{ass_J}--\ref{ass_delta} hold. Then, the sequence $(\bs x^k)_{k\in\NN}$ generated by Algorithm \ref{algo_i} with $F^{\textup{VR}}$ as in \eqref{eq_SAA} converges a.s. to a SNE of the game in \eqref{eq_game}.  
\end{theorem}
\begin{proof}
See Appendix \ref{app_theo_SAA}.
\end{proof}

%\begin{algorithm}[t]
%\caption{Extragradient (EG)}\label{EG}
%Initialization: $x_i^0 \in \Omega_i$\\
%Iteration $k$: Agent $i$ \\
%Receives $x_j^k$ for $j \neq i$, then updates:
%$$
%y_i^{k} =\op{proj}_{\Omega_i}[x^{k}-\alpha_{k}\hat F_{i}(x_i^k, x_j^k,\xi_i^k) ] 
%$$\\
%Receives $y_j^k$ for $j \neq i$, then updates:
%$$
%x_i^{k+1} =\op{proj}_{\Omega_i}[x^{k}-\alpha_{k} \hat F_{i}(y_i^k, y_j^k,\xi_i^k)]
%$$
%\end{algorithm}
%\begin{algorithm}[t]
%\caption{Adam}\label{adam}
%    \SetKwInOut{Input}{Input}
%    \SetKwInOut{Output}{Output}
%
%    \Input{Initial parameters $x^0,\bar x^0 \in \bs\Omega$\\
%    Exponential decay rates $\beta_1,\beta_2\in[0,1)$\\
%    Step size $\alpha$}
%    \Output{Parameters $x^{k+1}$}
%    \Init{}{$1^{\small{\text{st}}}$ moment vector $z^0=0$\\
%    $2^{\small{\text{nd}}}$ moment vector $y^0=0$\\
%    Time step k=0}
%    
%    \For{$k=1,\dots,K$}
%    {
%$\hat g^k=\hat F(x^k,\xi^k)$\quad \# update pseudogradient\\
%$z^k=\beta_1z^{k-1}+(1-\beta_1)\hat g^k$\quad \# update $1^{\small{\text{st}}}$ moment\\
%$y^k=\beta_2y^{k-1}+(1-\beta_2)(\hat g^k)^2$ \# update $2^{\small{\text{nd}}}$ moment\\
%$\tilde z^k=\frac{z^k}{1-\beta_1^k}$ \quad \# compute $1^{\small{\text{st}}}$ moment estimate\\
%$\tilde y^k=\frac{y^k}{1-\beta_2^k}$\quad \# compute $2^{\small{\text{nd}}}$ moment estimate\\
%$x^{k+1}=x^k-\alpha\frac{\tilde z^k}{\sqrt{\tilde y^k}+\epsilon}$\quad \# update parameters\\
%}
%\Return{$x^{K+1}$}
%\end{algorithm}
\subsection{Convergence of Algorithm \ref{algo_i2}}

In this section, we state the convergence result (and the required assumptions) for Algorithm \ref{algo_i2}.

First, the bound on the relaxation parameter is wider in this case (compared to Assumption \ref{ass_delta}).

\begin{assumption}\label{ass_delta1}
The relaxation parameter in Algorithm \ref{algo_i2} is such that $\delta\in[0,1]$.\fineass
\end{assumption}

Next, we postulate an assumption on the SA approximation in \eqref{eq_SA}, reasonable in our game theoretic framework \cite{gidel2019}. We also assume an explicit bound on the feasible set.

\begin{assumption}\label{ass_bounded}
$F^{\textup{SA}}$ in \eqref{eq_SA} is bounded, i.e., there exists $B>0$ such that for $\bs x\in \bs\Omega$, 
$\EE[\|F^{\textup{SA}}(\bs x,\xi)\|^2|\mc F_k]\leq B.$
\fineass\end{assumption}
 
\begin{assumption}\label{ass_boundedset}
The local constraint set $\bs\Omega$ is such that $\max\limits_{\bs x,\bs y\in \bs\Omega} \|\bs x-\bs y\|^{2} \leq R^{2}$, for some $R\geq0$.
\fineass\end{assumption}

To measure how close a point is to the solution, let us introduce the gap function,
\begin{equation}\label{gap}
\operatorname{err}(\bs x)=\max _{\bs x^* \in\bs \Omega} \langle\FF(\bs x^*),\bs x-\bs x^*\rangle,
\end{equation}
which is equal 0 if and only if $\bs x$ is a solution of the (S)VI in \eqref{eq_SVI} \cite[Eq. 1.5.2]{facchinei2007}. Other possible measure functions can be found in \cite{gidel2019}.

%EG %%%%%%%%%%%%%%%%%%%%%%%%%%%%%%
\begin{algorithm}[b]
\caption{Extragradient (EG)}\label{EG}
Initialization: $x_i^0 \in \Omega_i$\\
Iteration $k$: Agent $i$ \\
Receives $x_j^k$ for $j \neq i$, then updates:
$$
y_i^{k} =\op{proj}_{\Omega_i}[x^{k}-\alpha_{k}\hat F_{i}(x_i^k, x_j^k,\xi_i^k) ] 
$$
Receives $y_j^k$ for $j \neq i$, then updates:
$$
x_i^{k+1} =\op{proj}_{\Omega_i}[x^{k}-\alpha_{k} \hat F_{i}(y_i^k, y_j^k,\xi_i^k)]
$$
\end{algorithm}
% ADAM %%%%%%%%%%%%%%%%%%%%%%
\begin{algorithm}[t]
\caption{Adam}\label{adam}
    \SetKwInOut{Input}{Input}
    \SetKwInOut{Output}{Output}

    \Input{Initial parameters $x^0,\bar x^0 \in \bs\Omega$\\
    Exponential decay rates $\beta_1,\beta_2\in[0,1)$\\
    Step size $\alpha$}
    \Output{Parameters $x^{k+1}$}
    \Init{}{$1^{\small{\text{st}}}$ moment vector $z^0=0$\\
    $2^{\small{\text{nd}}}$ moment vector $y^0=0$\\
    Time step k=0}
    
    \For{$k=1,\dots,K$}
    {
$\hat g^k=\hat F(x^k,\xi^k)$\quad \# update pseudogradient\\
$z^k=\beta_1z^{k-1}+(1-\beta_1)\hat g^k$\quad \# update $1^{\small{\text{st}}}$ moment\\
$y^k=\beta_2y^{k-1}+(1-\beta_2)(\hat g^k)^2$ \# update $2^{\small{\text{nd}}}$ moment\\
$\tilde z^k=\frac{z^k}{1-\beta_1^k}$ \quad \# compute $1^{\small{\text{st}}}$ moment estimate\\
$\tilde y^k=\frac{y^k}{1-\beta_2^k}$\quad \# compute $2^{\small{\text{nd}}}$ moment estimate\\
$x^{k+1}=x^k-\alpha\frac{\tilde z^k}{\sqrt{\tilde y^k}+\epsilon}$\quad \# update parameters\\
}
\Return{$x^{K+1}$}
\end{algorithm}

We are now ready to state our second convergence result.
\begin{theorem}\label{theo_ave}
Let Assumptions \ref{ass_J}--\ref{ass_error} and \ref{ass_delta1}--\ref{ass_boundedset} hold. Let $\bs X^K=\frac{1}{K}\sum\nolimits_{k=1}^K\bs x^k$, $c=\frac{2-\delta^2}{1-\delta}$, $B$ be as in Assumption \ref{ass_bounded}, $R$ be as in Assumption \ref{ass_boundedset} and $\sigma^2$ be as in Assumption \ref{ass_error}.
Then, the sequence $(\bs x^k)_{k\in\NN}$ generated by Algorithm \ref{algo_i2} with constant step size and $F^{\textup{SA}}$ as in \eqref{eq_SA} is such that
$$\EE[\op{err}(\bs X^K)]=\frac{cR}{\lambda K}+(2B^2+\sigma^2)\lambda.$$
Thus,
$\lim_{K\to\infty}\EE[\op{err}(\bs X^K)]=(2B^2+\sigma^2)\lambda.$
\end{theorem}
\begin{proof}
See Appendix \ref{app_theo_ave}.
\end{proof}

\begin{remark}
The average defined in Theorem \ref{theo_ave} is not in conflict with the definition in \eqref{eq_ave} because if we consider a fixed step size, it holds that
$$\textstyle{\bs X^{K} = \frac{\sum_{k=1}^{K} \lambda_{k} \bs x^k}{\sum_{k=1}^{K} \lambda_{k}}=\frac{\lambda\sum_{k=1}^{K} \bs x^k}{K\lambda}=\frac{1}{K}\sum\nolimits_{k=1}^K\bs x^k}.\vspace{-.5cm}$$
\fineass\end{remark}

\section{Numerical simulations}

% RADAM %%%%%%%%%%%%%%%%%%%%%%%%%
\begin{algorithm}[t]
    \SetKwInOut{Input}{Input}
    \SetKwInOut{Output}{Output}

    \Input{Initial parameters $x^0,\bar x^0 \in \bs\Omega$\\
    Exponential decay rates $\beta_1,\beta_2\in[0,1)$\\
    Step size $\alpha$\\
    Relaxing parameter $\delta\in\left[\frac{2}{1+\sqrt{5}},1\right]$}
    \Output{Parameters $x^{k+1}$}
    \Init{}{$1^{\small{\text{st}}}$ moment vector $z^0=0$\\
    $2^{\small{\text{nd}}}$ moment vector $y^0=0$\\
    Time step k=0}
    
\For{$k=1,\dots,K$}
{
$\hat g^k=\hat F(x^k,\xi^k)$\quad \# update pseudogradient\\
$\bar{x}^{k} =(1-\delta) x^k+\delta\bar{x}^{k-1}$\quad \# relaxation step\\
$z^k=\beta_1z^{k-1}+(1-\beta_1)\hat g^k$\quad \# update $1^{\small{\text{st}}}$ moment\\
$y^k=\beta_2y^{k-1}+(1-\beta_2)(\hat g^k)^2$ \# update $2^{\small{\text{nd}}}$ moment\\
$\tilde z^k=\frac{z^k}{1-\beta_1^k}$ \quad \# compute $1^{\small{\text{st}}}$ moment estimate\\
$\tilde y^k=\frac{y^k}{1-\beta_2^k}$\quad \# compute $2^{\small{\text{nd}}}$ moment estimate\\
$x^{k+1}=\bar x^k-\alpha\frac{\tilde z^k}{\sqrt{\tilde y^k}+\epsilon}$\quad \# update parameters\\
}
\Return{$x^{K}$}
    \caption{Relaxed Adam}\label{RA}
\end{algorithm}

Let us present some numerical experiments to validate the analysis. We show how GANs are trained using our SRFB algorithm and we propose a comparison with one of the most used algorithms for GANs. 
Specifically, we compare our SRFB algorithm with the extragradient (EG) algorithm (Algorithm \ref{EG}) \cite{gidel2019}. We note that, compared to Algorithm \ref{algo_i}, Algorithm \ref{EG} involves two projection steps and two evaluations of the pseudogradient mapping. For the simulations we use Adam (Algorithm \ref{adam}) \cite{kingma2014} instead of the stochastic gradient \cite{robbins1951}. In Algorithm \ref{RA} we propose the Relaxed Adam, i.e., the SRFB algorithm with Adam; the EG algorithm with Adam can be derived similarly \cite[Algorithm 4]{gidel2019}.
All the simulations are performed on Matlab R2020a with 128G RAM and 2 * Intel(R) Xeon(R) Gold 6148 CPU @ 2.40GHz (20 cores each).

We train two DCGAN architectures \cite{wang2019,radford2015} (presented in Table \ref{tab_gan}) on the CIFAR10 dataset \cite{krizhevsky2009} with the GAN objective \cite{goodfellow2016,goodfellow2014}. We choose the hyperparameters of Adam as $\beta_1= 0.5$ and $\beta_2= 0.9$. We compute the inception score \cite{salimans2016} to have an objective comparison: the higher the inception score, the better the image generation. In Figure \ref{fig_inc_iter}, we show how the inception score increases with time; the solid lines represent a tracking average over the previous and following 50 values of the inceptions score, which is averaged over 20 runs. The transparent area indicated the maximum and minimum values obtained in the 20 runs.

We note that the SRFB algorithm is computationally less demanding than the EG algorithm. Specifically, in Figure \ref{fig_inc_iter}, after 24 hours (86400 seconds), the SRFB has performed approximatively 130000 iterations while the EG 90000. The averaged aSRFB shows worse performances (after approximatively 110000 iterations), but this is to be expected since we have convergence only to a neighborhood of the solution (Theorem \ref{theo_ave}). In Figure \ref{fig_snr}, we show the mean to variance ratio (average Inception Score divided by its variance) at each time instant of the three algorithms. As one can see, from Figure \ref{fig_inc_iter} and \ref{fig_snr} the SRFB algorithm has a similar performance to the EG algorithm but with a smaller variance, hence a higher reproducibility.

\begin{figure}[t]
%\begin{subfigure}{\columnwidth}
\centering
\includegraphics[width=\columnwidth]{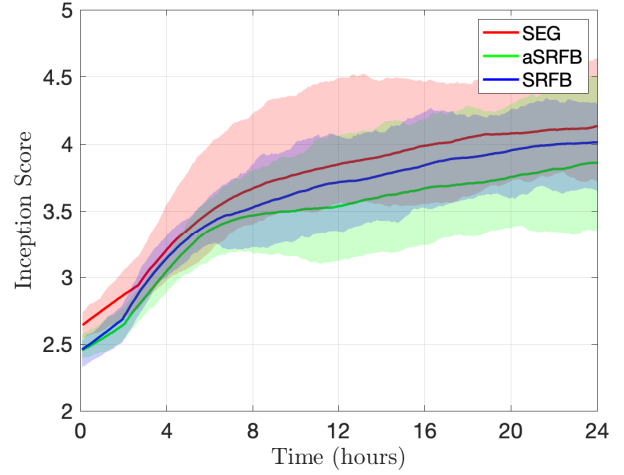}
\caption{Inception scores reached by the EG, the SRFB and the aSRFB  algorithms.}\label{fig_inc_iter}
%\end{subfigure}
\end{figure}

\begin{figure}[t]
%\begin{subfigure}{\columnwidth}
\centering
\includegraphics[width=\columnwidth]{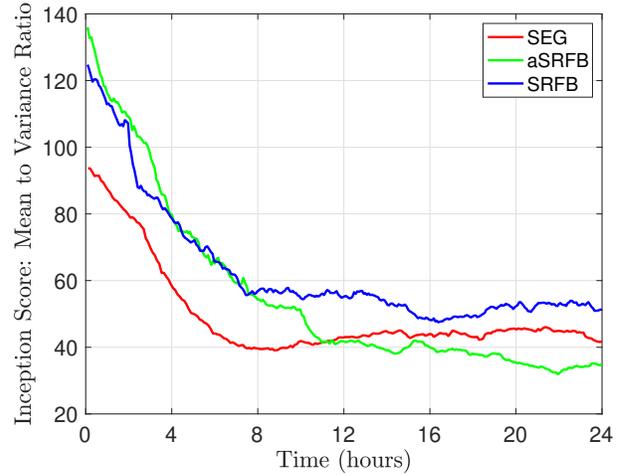}
\caption{Mean to variance ratio corresponding to the average inception scores.}\label{fig_snr}
%\end{subfigure}
\end{figure}

\section{Conclusion}
The stochastic relaxed forward-backward algorithm is a very promising algorithm for training Generative Adversarial Networks. If an increasing number of samples is available and the pseudogradient mapping of the game is monotone, convergence to the exact solution holds. Instead, with only a finite, fixed mini-batch and the same monotonicity assumption, convergence to a neighborhood of the solution can be proven by using an averaging technique. Our numerical experience shows a similar performance compared to the extragradient scheme, widely used in the literature for GANs.

For the future, it would be interesting to extend the convergence result to an exact solution also in the case of a small mini-batch. Since the cost function associated to GAN is often non-convex, it would also worth finding algorithms converging under weaker assumptions than monotonicity.

%\begin{figure}[h]
%\begin{subfigure}{\columnwidth}
%\centering
%\includegraphics[width=.7\textwidth]{Figs/RFB_300_3.png}
%\caption{}
%\end{subfigure}
%\begin{subfigure}{\columnwidth}
%\centering
%\includegraphics[width=.7\textwidth]{Figs/Avg_300.png}
%\caption{}
%\end{subfigure}
%%\end{figure}
%%\begin{figure}[h]
%\begin{subfigure}{\columnwidth}
%\centering
%\includegraphics[width=.7\textwidth]{Figs/EG_300_3.png}
%\caption{}
%\end{subfigure}
%\caption{Generated images with the SRFB algorithm (a),\tcb{ with the aSFRB algorithm (b)} and with the EG algorithm (c).}\label{fig_images}
%\end{figure}

\appendices
\section{Preliminary results}
% NN %%%%%%%%%%%%%%%%%%%%%%%%%%%
\begin{table}[t]
\begin{center}
\begin{tabular}{c}
\hline
Generator\\
\hline
Input $z\in\RR^{1\times 1\times100}$\\
Linear $100 \to 5\times 5 \times 512$\\
Transp. Conv (ker: $5\times 5$, $512 \to 256$, stride: 1, crop: 0)\\
Batch Normalization\\
ReLu\\
Transp. Conv (ker: $5\times 5$, $256 \to 128$, stride: 2, crop: 1)\\
Batch Normalization\\
ReLu\\
Transp. Conv (ker: $5\times 5$, $128 \to 3$, stride: 2, crop: 1)\\
Tanh($\cdot$)\\
\hline\\
%\end{tabular}\\
%    \vspace{.2cm}
%\begin{tabular}{p{\columnwidth}}
\hline
Discriminator\\
\hline
Input $z\in\RR^{64\times 64\times 3}$\\
Conv (ker: $5\times 5$, $3 \to 64$, stride: 2, pad: 1)\\
%Batch Normalization\\
LeakyReLu\\
Conv (ker: $5\times 5$, $64 \to 256$, stride: 2, pad: 1)\\
Batch Normalization\\
LeakyReLu\\
Conv (ker: $5\times 5$, $256 \to 512$, stride: 2, pad: 1)\\
Batch Normalization\\
LeakyReLu\\
Conv (ker: $5\times 5$, $512 \to 1$)\\
\hline
\end{tabular}
\caption{Neural networks used.}\label{tab_gan}
\end{center}
\end{table}

%\begin{figure}[t]
%\begin{subfigure}{\columnwidth}
%\centering
%\includegraphics[width=.7\textwidth]{Figs/RFB_300_3.png}
%\caption{}
%\end{subfigure}
%\begin{subfigure}{\columnwidth}
%\centering
%\includegraphics[width=.7\textwidth]{Figs/Avg_300.png}
%\caption{}
%\end{subfigure}
%%\end{figure}
%%\begin{figure}[h]
%\begin{subfigure}{\columnwidth}
%\centering
%\includegraphics[width=.7\textwidth]{Figs/EG_300_3.png}
%\caption{}
%\end{subfigure}
%\caption{Generated images with the SRFB algorithm (a),\tcb{ with the aSFRB algorithm (b)} and with the EG algorithm (c).}\label{fig_images}
%\end{figure}

We here recall some facts about norms, some properties of the projection operator and a preliminary result. Some results find inspiration from \cite{malitsky2019} where the algorithm is presented in the deterministic case.
We start with the norms. 
We use the cosine rule 
\begin{equation}\label{cosine}
\textstyle{\langle x,y\rangle=\frac{1}{2}(\langle x,x\rangle+\langle y,y\rangle-\norm{x-y}^2)}
\end{equation}
and the following two property of the norm \cite[Corollary 2.15]{bau2011}, $\forall a, b \in \mathcal{E}$, $\forall \alpha \in \mathbb{R}$
\begin{equation}\label{norm_conv}
\|\alpha a+(1-\alpha) b\|^{2}=\alpha\|a\|^{2}+(1-\alpha)\|b\|^{2}-\alpha(1-\alpha)\|a-b\|^{2},
\end{equation}
\begin{equation}\label{norm_sum}
\|a+b\|^{2}\leq2\|a\|^{2}+2\|b\|^{2}.
\end{equation}

Concerning the projection operator, by \cite[Proposition 12.26]{bau2011}, it satisfies the following inequality: let $C$ be a nonempty closed convex set, then, for all $x,y\in C$
\begin{equation}\label{proj}
\bar x=\op{proj}_C(x)\Leftrightarrow\langle \bar x-x, y-\bar x\rangle \geq 0 .
\end{equation}
The projection is also firmly non expansive \cite[Prop. 4.16]{bau2011}, and consequently, quasi firmly non expansive \cite[Def. 4.1]{bau2011}.

The Robbins-Siegmund Lemma is widely used in literature to prove a.s. convergence of sequences of random variables.
\begin{lemma}[Robbins-Siegmund Lemma, \cite{RS1971}]\label{lemma_RS}
Let $\mc F=(\mc F_k)_{k\in\NN}$ be a filtration. Let $\{\alpha_k\}_{k\in\NN}$, $\{\theta_k\}_{k\in\NN}$, $\{\eta_k\}_{k\in\NN}$ and $\{\chi_k\}_{k\in\NN}$ be non negative sequences such that $\sum_k\eta_k<\infty$, $\sum_k\chi_k<\infty$ and let
$$\forall k\in\NN, \quad \EE[\alpha_{k+1}|\mc F_k]+\theta_k\leq (1+\chi_k)\alpha_k+\eta_k \quad a.s.$$
Then $\sum_k \theta_k<\infty$ and $\{\alpha_k\}_{k\in\NN}$ converges a.s. to a non negative random variable.\fineass
\end{lemma}

The next lemma collects some properties that follow from the definition of the SRFB algorithm.
\begin{lemma}\label{lemma_algo}
Given Algorithm \ref{algo_i}, the following statements hold.
\begin{enumerate}
\item $\bs x^k-\bar {\bs x}^{k-1}=\frac{1}{\delta}(\bs x^k-\bar {\bs x}^k)$
\item $\bs x^{k+1}-\bs x^*=\frac{1}{1-\delta}(\bar{\bs x}^{k+1}-\bs x^*)-\frac{\delta}{1-\delta}(\bar{\bs x}^k-\bs x^*)$
\item $\frac{\delta}{(1-\delta)^2}\norm{\bar{\bs x}^{k+1}-{\bs x}^k}^2=\delta\norm{\bs x^{k+1}-{\bs x}^k}^2$
\end{enumerate}
\end{lemma}
\begin{proof}
Straightforward from Algorithm \ref{algo_i} and \cite{malitsky2019}.
\end{proof}

\begin{figure}[H]
\begin{subfigure}{\columnwidth}
\centering
\includegraphics[width=.8\textwidth]{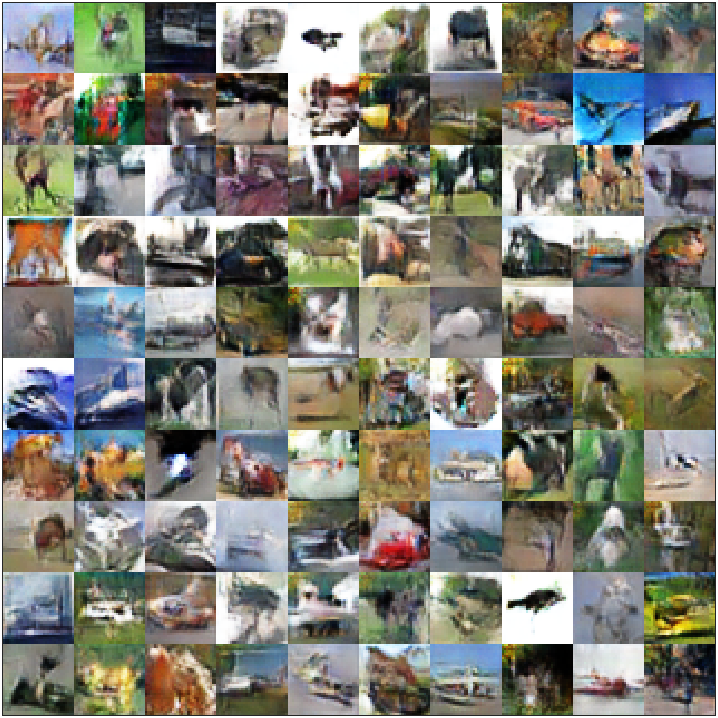}
\caption{}
\end{subfigure}
\begin{subfigure}{\columnwidth}
\centering
\includegraphics[width=.8\textwidth]{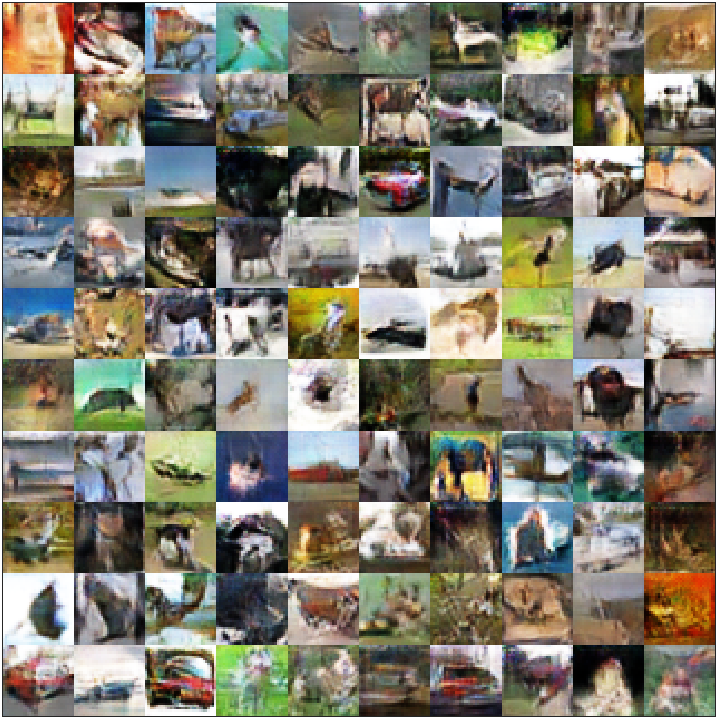}
\caption{}
\end{subfigure}
%\end{figure}
%\begin{figure}[h]
\begin{subfigure}{\columnwidth}
\centering
\includegraphics[width=.8\textwidth]{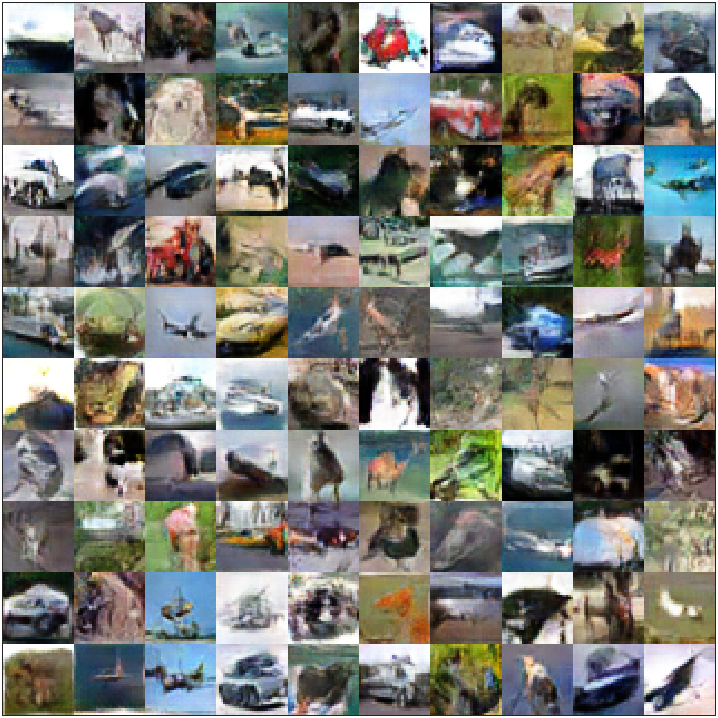}
\caption{}
\end{subfigure}
\caption{Generated images with the SRFB algorithm (a), with the aSFRB algorithm (b) and with the EG algorithm (c).}\label{fig_images}
\end{figure}

\section{Proof of Theorem \ref{theo_SAA}}\label{app_theo_SAA}

\begin{proof}[Proof of Theorem \ref{theo_SAA}]
Using the property of projection operator (\ref{proj}) we have
\begin{equation}\label{stepNk1}
\langle \bs x^{k+1}-\bar {\bs x}^k +\lambda F^{\textup{VR}}(\bs x^k,\xi^k),\bs x^*-\bs x^{k+1}\rangle\geq0,
\end{equation}
\begin{equation}\label{stepNk2}
\langle \bs x^k-\bar {\bs x}^{k-1} +\lambda F^{\textup{VR}}(\bs x^{k-1},\xi^{k-1}),\bs x^{k+1}-\bs x^k\rangle\geq0.
\end{equation}
Using Lemma \ref{lemma_algo}.1, (\ref{stepNk2}) becomes
\begin{equation}\label{stepNk3}
\langle \frac{1}{\delta}(\bs x^k-\bar {\bs x}^k)+\lambda F^{\textup{VR}}(\bs x^{k-1},\xi^{k-1}),\bs x^{k+1}-\bs x^k\rangle\geq 0.
\end{equation}
Then, adding (\ref{stepNk1}) and (\ref{stepNk3}) we obtain
\begin{equation}\label{stepNk4}
\begin{aligned}
&\langle \bs x^{k+1}-\bar {\bs x}^k +\lambda F^{\textup{VR}}(\bs x^k,\xi^k),\bs x^*-\bs x^{k+1}\rangle+\\
&+\langle \frac{1}{\delta}(\bs x^k-\bar {\bs x}^k)+\lambda F^{\textup{VR}}(\bs x^{k-1},\xi^{k-1}),\bs x^{k+1}-\bs x^k\rangle\geq0.
\end{aligned}
\end{equation}
Now we use the cosine rule (\ref{cosine}),
\begin{equation*}\label{cos_zeta}
\begin{aligned}
&\langle \bs x^{k+1}-\bar {\bs x}^k,\bs x^*-\bs x^{k+1}\rangle=\\
&\textstyle{-\frac{1}{2}(\norm{\bs x^{k+1}-\bar {\bs x}^k}^2+\norm{\bs x^{k+1}-\bs x^*}^2-\norm{\bs x^*-\bar {\bs x}^k}^2)}\\
&\textstyle{\langle \frac{1}{\delta}(\bs x^k-\bar {\bs x}^k),\bs x^{k+1}-\bs x^k\rangle=}\\
&\textstyle{-\frac{1}{2\delta}(\norm{\bs x^k-\bar {\bs x}^k}^2+\norm{\bs x^k-\bs x^{k+1}}^2-\norm{\bs x^{k+1}-\bar {\bs x}^k}^2)}
\end{aligned}
\end{equation*}
and we note that
\begin{equation*}\label{term_in_F}
\begin{aligned}
&\lambda\langle F^{\textup{VR}}(\bs x^k,\xi^k),\bs x^*-\bs x^{k+1}\rangle\\
&=-\lambda\langle\FF(\bs x^k),\bs x^k-\bs x^*\rangle+\langle \epsilon^k,\bs x^*-\bs x^k\rangle+\\
&+\lambda\langle\FF(\bs x^k),\bs x^k-\bs x^{k+1}\rangle+\langle \epsilon^k,\bs x^k-\bs x^{k+1}\rangle.
\end{aligned}
\end{equation*}

Then, by reordering and substituting in (\ref{stepNk4}), we obtain:
\begin{equation}\label{stepNk7}
\begin{aligned}
-&\norm{\bs x^{k+1}-\bar {\bs x}^k}^2-\norm{\bs x^{k+1}-\bs x^*}^2+\norm{\bs x^*-\bar {\bs x}^k}^2+\\
-&\textstyle{\frac{1}{\delta}}\norm{\bs x^k-\bar {\bs x}^k}^2-\textstyle{\frac{1}{\delta}}\norm{\bs x^k-\bs x^{k+1}}^2+\textstyle{\frac{1}{\delta}}\norm{\bs x^{k+1}-\bar {\bs x}^k}^2+\\
-&2\lambda\langle\FF(\bs x^k),\bs x^k-\bs x^*\rangle+2\lambda\langle\varepsilon^k,\bs x^*-\bs x^k\rangle+\\
+&2\lambda\langle \FF(\bs x^k)-\FF(\bs x^{k-1}),\bs x^k-\bs x^{k+1}\rangle+\\
+&2\lambda\langle \varepsilon^k-\varepsilon^{k-1},\bs x^k-\bs x^{k+1}\rangle\geq 0.
\end{aligned}
\end{equation}
Since $\FF$ is monotone, it holds that
%\begin{equation}\label{mono}
$\langle\FF(\bs x^k),\bs x^k-\bs x^*\rangle\geq\langle\FF(\bs  x^*),\bs x^k-\bs x^*\rangle\geq0.$
%\end{equation}
%\tcb{Now we apply Lemma \ref{lemma_algo}.2 and Lemma \ref{lemma_algo}.3 to $\norm{\bs x^{k+1}-\bs x^*}$:
%\begin{equation}\label{step_delta}
%\begin{aligned}
%\norm{\bs x^{k+1}-\bs x^*}^2
%=&\frac{1}{1-\delta}\norm{\bar{\bs x}^{k+1}-\bs x^*}^2-\frac{\delta}{1-\delta}\norm{\bar{\bs x}^k-\bs x^*}^2+\\
%&+\delta\norm{\bs x^{k+1}-{\bs x}^k}^2.
%\end{aligned}
%\end{equation}}
By using Lemma \ref{lemma_algo}.2 and \ref{lemma_algo}.3 as in (\ref{step_delta}), and substituting in (\ref{stepNk7}), 
grouping and reordering, we get
\begin{equation}\label{stepNk9}
\begin{aligned}
&\textstyle{\frac{1}{1-\delta}\norm{\bar{\bs x}^{k+1}-\bs x^*}^2+\frac{1}{\delta}\norm{\bs x^k-\bs x^{k+1}}^2}\\
&\leq\textstyle{ \left(\frac{\delta}{1-\delta}+1\right)\norm{\bs x^*-\bar {\bs x}^k}^2-\frac{1}{\delta}\norm{\bs x^k-\bar {\bs x}^k}^2+}\\
&+2\lambda\langle \FF(\bs x^k)-\FF(\bs x^{k-1}),\bs x^k-\bs x^{k+1}\rangle\\
&+2\lambda\langle\varepsilon^k,\bs x^*-\bs x^k\rangle+2\lambda\langle \varepsilon^k-\varepsilon^{k-1},\bs x^k-\bs x^{k+1}\rangle
\end{aligned}
\end{equation}
where we used Assumption \ref{ass_delta}. 
Moreover, by using Lipschitz continuity of $\FF$ and Cauchy-Schwartz and Young's inequality, it follows that
\begin{equation*}
\begin{aligned}
&\lambda\langle \FF(\bs x^k)-\FF(\bs x^{k-1}),\bs x^k-\bs x^{k+1}\rangle \\
&\leq\textstyle{\frac{\ell\lambda}{2}}(\normsq{\bs x^k-\bs x^{k-1}}+\normsq{\bs x^k-\bs x^{k+1}}).
\end{aligned}
\end{equation*}
Similarly we can bound the term involving the stochastic errors:
\begin{equation*}
\begin{aligned}
&2\lambda\langle \varepsilon^k-\varepsilon^{k-1},\bs x^k-\bs x^{k+1}\rangle\\
&\leq2\lambda\norm{\varepsilon^k-\varepsilon^{k-1}}\norm{\bs x^k-\bs x^{k+1}}\\
&\leq\lambda\normsq{\varepsilon^k-\varepsilon^{k-1}}+\lambda\normsq{\bs x^k-\bs x^{k+1}}.
\end{aligned}
\end{equation*}
By substituting in (\ref{stepNk9}), we conclude that
\begin{equation}\label{step_final}
\begin{aligned}
&\textstyle{\frac{1}{1-\delta}\norm{\bar{\bs x}^{k+1}-\bs x^*}^2+\frac{1}{\delta}\norm{\bs x^k-\bs x^{k+1}}^2}\\
&\leq\textstyle{\frac{1}{1-\delta}\norm{\bs x^*-\bar {\bs x}^k}^2-\frac{1}{\delta}\norm{\bs x^k-\bar {\bs x}^k}^2+}\\
&+\ell\lambda\left(\normsq{\bs x^k-\bs x^{k-1}}+\normsq{\bs x^k-\bs x^{k+1}}\right)+\\
&+ \lambda\normsq{\varepsilon^k-\varepsilon^{k-1}}+\lambda\normsq{\bs x^k-\bs x^{k+1}}\\
&+2\lambda\langle\varepsilon^k,\bs x^*-\bs x^k\rangle
\end{aligned}
\end{equation}
Now, we consider the residual function of $\bs x^k$:
\begin{equation*}
\begin{aligned}
\op{res}(\bs x^k)^2&=\normsq{\bs x^k-\op{proj}(\bs x^k-\lambda\FF(\bs x^k)}\\
&\leq 2\normsq{\bs x^k-\bs x^{k+1}}+2\normsq{\bar{\bs x}_k-\bs x^k+\lambda\varepsilon_k}\\
&\leq2\normsq{\bs x^k-\bs x^{k+1}}+4\normsq{\bar{\bs x}_k-\bs x^k}+4\lambda^2\normsq{\varepsilon_k}
\end{aligned}
\end{equation*}
where we added and subtracted $\bs x^{k+1}=\op{proj}(\bar{\bs x}_k-\lambda F^{\textup{VR}}(\bs x^k))$ in the first inequality and used the firmly non expansiveness of the projection and \eqref{norm_sum}.
It follows that
\begin{equation*}
\textstyle{\normsq{\bar{\bs x}_k-\bs x^k}\geq\frac{1}{4}\op{res}(\bs x^k)^2-\frac{1}{2}\normsq{\bs x^k-\bs x^{k+1}}-\lambda^2\normsq{\varepsilon_k}}
\end{equation*}
By substituting in \eqref{step_final}, we have that
\begin{equation*}
\begin{aligned}
&\textstyle{\frac{1}{1-\delta}\norm{\bar{\bs x}^{k+1}-\bs x^*}^2+\frac{1}{\delta}\norm{\bs x^k-\bs x^{k+1}}^2\leq\frac{1}{1-\delta}\norm{\bs x^*-\bar {\bs x}^k}^2}\\
&\textstyle{-\frac{1}{\delta}\left(\frac{1}{4}\op{res}(\bs x^k)^2-\frac{1}{2}\normsq{\bs x^k-\bs x^{k+1}}-\lambda^2\normsq{\varepsilon_k}\right)}\\
&+\ell\lambda\left(\normsq{\bs x^k-\bs x^{k-1}}+\normsq{\bs x^k-\bs x^{k+1}}\right)\\
&+ \lambda\normsq{\varepsilon^k-\varepsilon^{k-1}}+\lambda\normsq{\bs x^k-\bs x^{k+1}}+2\lambda\langle\varepsilon^k,\bs x^*-\bs x^k\rangle.
\end{aligned}
\end{equation*}
Finally, by taking the expected value, grouping and using Remark \ref{remark_error} and Assumptions \ref{ass_error} and \ref{ass_delta}, we have
\begin{equation*}\label{stepNk10}
\begin{aligned}
&\textstyle{\EEk{\frac{1}{1-\delta}\norm{\bar{\bs x}^{k+1}-\bs x^*}^2}}+\\
&\textstyle{+\EEk{\left(\frac{1}{2\delta}-\ell\lambda-\lambda\right)\norm{\bs x^k-\bs x^{k+1}}^2}}\\
&\textstyle{\leq\frac{1}{1-\delta}\norm{\bs x^*-\bar {\bs x}^k}^2+\ell\lambda\norm{\bs x^k-\bs x^{k-1}}^2+}\\
&\textstyle{+\frac{2\lambda C\sigma}{N_k}+\frac{2\lambda C\sigma}{N_{k-1}}+\frac{\lambda}{\delta}\frac{C\sigma}{N_k}}\\
&\textstyle{-\frac{1}{\delta}\norm{\bs x^k-\bar {\bs x}^k}^2-\frac{1}{4\delta}\op{res}(\bs x^k)^2.}\\
\end{aligned}
\end{equation*}
To use Lemma \ref{lemma_RS}, let
$\alpha_k= \frac{1}{1-\delta}\norm{\bs x^*-\bar {\bs x}^k}^2+\ell\lambda\norm{\bs x^k-\bs x^{k-1}}^2,$
$\theta_k=\frac{1}{\delta}\norm{\bs x^k-\bar {\bs x}^k}^2+\frac{1}{4\delta}\op{res}(\bs x^k)^2,$
$\eta_k=\frac{2\lambda C\sigma}{N_k}+\frac{2\lambda C\sigma}{N_{k-1}}+\frac{\lambda}{\delta}\frac{C\sigma}{N_k}.$
Applying the Robbins Siegmund Lemma we conclude that $\alpha_k$ converges and that $\sum_{k\in\NN}\theta_k$ is summable. This implies that the sequence $(\bar {\bs x}^k)_{k\in\NN}$ is bounded and that $\|\bs x^k-\bar {\bs x}^k\|\to 0$ (otherwise $\sum \frac{1}{\delta}\|\bs x^k-\bar {\bs x}^k\|^2=\infty$). Therefore $(\bs x^k)_{k\in\NN}$ has at least one cluster point $\tilde{\bs x}$. Moreover, since $\sum_{k\in\NN}\theta_k<\infty$, $\op{res}(\bs x^k)^2\to0$ and $\op{res}(\tilde{\bs x}^k)^2=0$.
\end{proof}

\section{Proof of Theorem \ref{theo_ave}}\label{app_theo_ave}

\begin{proof}[Proof of Theorem \ref{theo_ave}]
We start by using the fact that the projection is firmly quasinonexpansive.
\begin{equation*}\label{sdet1}
\begin{aligned}
&\normsq{\bs x^{k+1}-\bs x^*}\\
\leq&\normsq{\bs x^*-\tilde{\bs x}^k+\lambda F^{\textup{SA}}(\bs x^k,\xi^k)}\\
&-\normsq{\bar {\bs x}^k-\lambda F^{\textup{SA}}(\bs x^k,\xi^k)-\bs x^{k+1}}\\
\leq&\normsq{\bs x-\bar{\bs x}^k}-\normsq{\bar{\bs x}^k-\bs x^{k+1}}+2\lambda_k\langle F^{\textup{SA}}(\bs x^k,\xi^k),\bs x^*-\bar{\bs x}^k\rangle\\
&+2\lambda_k\langle F^{\textup{SA}}(\bs x^k,\xi^k),\bar{\bs x}^k-\bs x^{k+1}\rangle\\
=&\normsq{\bs x^*-\bar{\bs x}^k}-\normsq{\bar{\bs x}^k-\bs x^{k+1}}+2\lambda_k\langle F^{\textup{SA}}(\bs x^k,\xi^k),\bar{\bs x}^k-\bs x^{k+1}\rangle\\
&+2\lambda_k\langle F^{\textup{SA}}(\bs x^k,\xi^k),\bs x^*-\bs x^k\rangle+2\lambda_k\langle F^{\textup{SA}}(\bs x^k,\xi^k),\bs x^k-\bar{\bs x}^k\rangle\\
\end{aligned}
\end{equation*}
Now we apply Lemma \ref{lemma_algo}.2 and Lemma \ref{lemma_algo}.3 to $\norm{\bs x^{k+1}-\bs x^*}$:
\begin{equation}\label{step_delta}
\begin{aligned}
\norm{\bs x^{k+1}-\bs x^*}^2
=&\textstyle{\frac{1}{1-\delta}\norm{\bar{\bs x}^{k+1}-\bs x^*}^2-\frac{\delta}{1-\delta}\norm{\bar{\bs x}^k-\bs x^*}^2+}\\
&+\delta\norm{\bs x^{k+1}-{\bs x}^k}^2.
\end{aligned}
\end{equation}
Then, we can rewrite the inequality as
\begin{equation}\label{sdet2}
\begin{aligned}
&\textstyle{\frac{1}{1-\delta}\normsq{\bar{\bs x}^{k+1}-\bs x^*}\leq\frac{1}{1-\delta}\normsq{\bar{\bs x}^k-\bs x^*}}\\
&+2\lambda_k\langle F^{\textup{SA}}(\bs x^k,\xi^k),\bs x^*-\bs x^k\rangle+2\lambda_k\langle F^{\textup{SA}}(\bs x^k,\xi^k),\bs x^k-\bar{\bs x}^k\rangle\\
&+2\lambda_k\langle F^{\textup{SA}}(\bs x^k,\xi^k),\bar{\bs x}^k-\bs x^{k+1}\rangle-(\delta+1)\normsq{\bs x^{k+1}-\bar{\bs x}^k}\\
\end{aligned}
\end{equation}
By applying Young's inequality we obtain
\begin{equation*}
\begin{aligned}
2\lambda_k&\langle F^{\textup{SA}}(\bs x^k,\xi^k),\bs x^k-\bar{\bs x}^k\rangle\\
&\leq\lambda_k^2\normsq{F^{\textup{SA}}(\bs x^k,\xi^k)}+\normsq{\bs x^k-\bar{\bs x}^k}\\
2\lambda_k&\langle F^{\textup{SA}}(\bs x^k,\xi^k),\bar{\bs x}^k-\bs x^{k+1}\rangle\\
&\leq\lambda_k^2\normsq{F^{\textup{SA}}(\bs x^k,\xi^k)}+\normsq{\bar{\bs x}^k-\bs x^{k+1}}
\end{aligned}
\end{equation*}
Then, inequality (\ref{sdet2}) becomes
\begin{equation}\label{sdet3}
\begin{aligned}
&\textstyle{\frac{1}{1-\delta}\normsq{\bar{\bs x}^{k+1}-\bs x^*}
\leq\frac{1}{1-\delta}\normsq{\bar{\bs x}^k-\bs x^*}}\\
&+2\lambda_k\langle F^{\textup{SA}}(\bs x^k,\xi^k),\bs x^*-\bs x^k\rangle\\
&+2\lambda_k^2\normsq{F^{\textup{SA}}(\bs x^k,\xi^k)}-(\delta+1)\normsq{\bs x^{k+1}-\bar{\bs x}^k}\\
&+\normsq{\bs x^k-\bar{\bs x}^k}+\normsq{\bar{\bs x}^k-\bs x^{k+1}}\\
\end{aligned}
\end{equation}
Reordering, adding and subtracting $2\lambda_k\langle\FF(\bs x^k),\bs x^k-\bs x^*\rangle$ and using Lemma \ref{lemma_algo}, we obtain
\begin{equation}\label{sdet3}
\begin{aligned}
&\textstyle{\frac{1}{1-\delta}\normsq{\bar{\bs x}^{k+1}-\bs x^*}+\delta\normsq{\bs x^{k+1}-\bar{\bs x}^k}\leq\frac{1}{1-\delta}\normsq{\bar{\bs x}^k-\bs x^*}}\\
&+2\lambda_k\langle\FF(\bs x^k)-F^{\textup{SA}}(\bs x^k,\xi^k),\bs x^k-\bs x^*\rangle-2\lambda_k\langle\FF(\bs x^k),\bs x^k-\bs x^*\rangle\\
&+2\lambda_k^2\normsq{F^{\textup{SA}}(\bs x^k,\xi^k)}+\delta^2\normsq{\bs x^k-\bs x^{k-1}}\\
\end{aligned}
\end{equation}
Then, by the definition of $\epsilon^k$, reordering leads to
\begin{equation}\label{eq_step}
\begin{aligned}
2\lambda_k&\langle\FF(\bs x^k),\bs x^k-\bs x^*\rangle\\
\leq&\textstyle{\frac{1}{1-\delta}}(\normsq{\bar{\bs x}^k-\bs x^*}-\normsq{\bar{\bs x}^{k+1}-\bs x^*})+\\
&+\delta(\normsq{\bs x^k-\bs x^{k-1}}-\normsq{\bs x^{k+1}-\bar{\bs x}^k})\\
&+2\lambda_k^2\normsq{F^{\textup{SA}}(\bs x^k,\xi^k)}+2\lambda_k\langle\epsilon^k,\bs x^k-\bs x^*\rangle
\end{aligned}
\end{equation}
Next, we sum over all the iterations, hence inequality \eqref{eq_step} becomes
\begin{equation}\label{sdet3}
\begin{aligned}
2\sum\nolimits_{k=1}^K\lambda_k&\langle\FF(\bs x^k),\bs x^k-\bs x^*\rangle\leq2\sum\nolimits_{k=1}^K\lambda_k\langle\epsilon^k,\bs x^k-\bs x^*\rangle\\
\leq&\textstyle{\frac{1}{1-\delta}\sum\nolimits_{k=1}^K(\normsq{\bar{\bs x}^k-\bs x^*}-\normsq{\bar{\bs x}^{k+1}-\bs x^*})+}\\
&+\textstyle{\delta\sum\nolimits_{k=1}^K(\normsq{\bs x^k-\bs x^{k-1}}-\normsq{\bs x^{k+1}-\bar{\bs x}^k})}+\\
&+\textstyle{2\sum\nolimits_{k=1}^K\lambda_k^2\normsq{F^{\textup{SA}}(\bs x^k,\xi^k)}}
\end{aligned}
\end{equation}
Using Assumption \ref{ass_mono} and resolving the sums, we obtain
\begin{equation}\label{step}
\begin{aligned}
2\sum\nolimits_{k=1}^K\lambda_k&\langle\FF(\bs x^*),\bs x^k-\bs x^*\rangle\leq2\sum\nolimits_{k=1}^K\lambda_k\langle\epsilon^k,\bs x^k-\bs x^*\rangle\\
\leq&\textstyle{\frac{1}{1-\delta}\normsq{\bar{\bs x}^0-\bs x^*}+\delta\normsq{\bs x^{0}-\bar{\bs x}^{-1}}+}\\
&+2\sum\nolimits_{k=1}^K\lambda_k^2\normsq{F^{\textup{SA}}(\bs x^k,\xi^k)}
\end{aligned}
\end{equation}
Now, we note that $\langle\epsilon^k,\bs x^k-\bs x^*\rangle=\langle\epsilon^k,\bs x^k-\bs u^k\rangle+\langle\epsilon^k,\bs u^k-\bs x^*\rangle.$
We define $\bs u^0=\bs x^0$ and $\bs u^{k+1}=\op{proj}(\bs u^k-\lambda_k\epsilon^k)$, thus
\begin{equation}
\begin{aligned}
\|&\bs u^{k+1}-\bs x^*\|^2=\normsq{\op{proj}(\bs u^k-\lambda_k\epsilon^k)-\bs x^*}\\
&\leq\normsq{\bs u^k-\lambda_k\epsilon^k-\bs x^*}\\
&\leq\normsq{\bs u^k-\bs x^*}+\lambda_k\normsq{\epsilon^k}-2\lambda_k\langle\epsilon^k,\bs u^k-\bs x^*\rangle
\end{aligned}
\end{equation}
Therefore,
%\begin{equation*}
%\begin{aligned}
$2\lambda_k\langle\epsilon^k,\bs x^k-\bs x^*\rangle=2\lambda_k\langle\epsilon^k,\bs x^k-\bs u^k\rangle+\normsq{\bs u^k-\bs x^*}+\lambda_k\normsq{\epsilon^k}-\normsq{\bs u^{k+1}-\bs x^*}.$
%\end{aligned}
%\end{equation*}
By including this in \eqref{step} and by doing the sum, we obtain
\begin{equation}\label{step2}
\begin{aligned}
2\sum\nolimits_{k=1}^K\lambda_k&\langle\FF(\bs x^*),\bs x^k-\bs x^*\rangle\\
\leq&\textstyle{\frac{1}{1-\delta}\normsq{\bar{\bs x}^0-\bs x^*}+\delta\normsq{\bs x^{0}-\bar{\bs x}^{-1}}+}\\
&+2\sum\nolimits_{k=1}^K\lambda_k^2\normsq{F^{\textup{SA}}(\bs x^k,\xi^k)}+\sum\nolimits_{k=1}^K \lambda_k^2\normsq{\epsilon^k}\\
&+\normsq{\bs u_0-\bs x^*}+2\sum\nolimits_{k=1}^K \lambda_k\langle\epsilon^k,\bs x^k-\bs u^k\rangle
\end{aligned}
\end{equation}
By definition, $\normsq{\bs u^0-\bs x^*}=\normsq{\bs x^0-\bs x^*}$. Then, tby aking the expected value in \eqref{step2} and using Assumption \ref{ass_error}, we conclude that
\begin{equation}
\begin{aligned}
2\sum\nolimits_{k=1}^K\lambda_k&\langle\FF(\bs x^*),\bs x^k-\bs x^*\rangle\\
\leq&\textstyle{\left(\frac{1}{1-\delta}+1\right)\normsq{\bar{\bs x}^0-\bs x^*}+\delta\normsq{\bs x^{0}-\bar{\bs x}^{-1}}}+\\
&+2\sum\nolimits_{k=1}^K\lambda_k^2\EEk{\normsq{F^{\textup{SA}}(\bs x^k,\xi^k)}}+\\
&+\sum\nolimits_{k=1}^K\lambda_k^2\EEk{\normsq{\epsilon^k}}.
\end{aligned}
\end{equation}
Let us define $S=\sum\nolimits_{k=1}^K\lambda_k$, $\bs X^K=\frac{\sum\nolimits_{k=1}^K\lambda_k\bs x^k}{\sum\nolimits_{k=1}^K\lambda_k}=\frac{1}{S}\sum\nolimits_{k=1}^K\lambda_k\bs x^k$. Then,
\begin{equation}
\begin{aligned}
2S&\langle\FF(\bs x^*),\bs X^K-\bs x^*\rangle\\
\leq&\textstyle{\frac{2-\delta}{1-\delta}\normsq{\bar{\bs x}^0-\bs x^*}+\delta\normsq{\bs x^{0}-\bar{\bs x}^{-1}}+}\\
&+2\sum\nolimits_{k=1}^K\lambda_k^2\EEk{\normsq{F^{\textup{SA}}(\bs x^k,\xi^k)}}+\\
&+\sum\nolimits_{k=1}^K\lambda_k^2\EEk{\normsq{\epsilon^k}}\\
&\textstyle{\leq\frac{2-\delta^2}{1-\delta}R+(2B^2+\sigma^2)\sum\nolimits_{k=1}^K\lambda_k.}
\end{aligned}
\end{equation}
Finally, it holds that if $\lambda_k$ is constant, $S=K\lambda$ and $\sum\nolimits_{k=1}^K\lambda_k^2=K\lambda^2$
\begin{equation*}
\langle\FF(\bs x^*),\bs X^K-\bs x^*\rangle\leq \frac{cR}{K\lambda}+(2B^2+\sigma^2)\lambda.\vspace{-.6cm}
\end{equation*}

\end{proof}

\ifCLASSOPTIONcaptionsoff
  \newpage
\fi

\bibliographystyle{IEEEtran}
%\bibliography{IEEEabrv,Biblio}
\bibliography{IEEEabrv,Biblio_abrv}
\balance

\begin{IEEEbiography}[{\includegraphics[scale=.37]{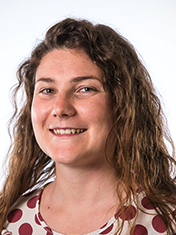}}]{Barbara Franci} is a PostDoc at the Delft Center for Systems and Control, Delft University of Technology, Delft, The Netherlands. Her current research interests are on Game Theory and its applications.

She received the Bachelor's and Master's degree in Pure Mathematics from University of Siena, Siena, Italy, respectively in 2012 and 2014. Then, she received her PhD from Politecnico of Turin and University of Turin, Turin, Italy, in 2018. In September - December 2016 she visited the Department of Mechanical Engineering, University of California, Santa Barbara, USA.
She was awarded in 2017 with the PhD Quality Award by the Academic Board of Politecnico di Torino.
\end{IEEEbiography}

\begin{IEEEbiography}[{\includegraphics[scale=.37]{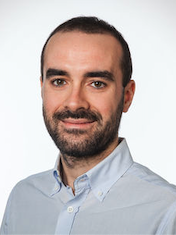}}]{Sergio Grammatico} is an Associate Professor at the Delft Center for Systems and Control, Delft University of Technology, The Netherlands. 

Born in Italy, in 1987, he received the Bachelor's degree (summa cum laude) in Computer Engineering, the Master's degree (summa cum laude) in Automatic Control Engineering, and the Ph.D. degree in Automatic Control, all from the University of Pisa, Italy, in February 2008, October 2009, and March 2013 respectively. He also received a Master's degree (summa cum laude) in Engineering Science from the Sant'Anna School of Advanced Studies, Pisa, Italy, in November 2011. 
In February-April 2010 and in November-December 2011, he visited the Department of Mathematics, University of Hawaii at Manoa, USA; in January-July 2012, he visited the Department of Electrical and Computer Engineering, University of California at Santa Barbara, USA. 

In 2013-2015, he was a post-doctoral Research Fellow in the Automatic Control Laboratory, ETH Zurich, Switzerland. In 2015-2018, he was an Assistant Professor, first in the Department of Electrical Engineering, Control Systems, TU Eindhoven, The Netherlands, then at the Delft Center for Systems and Control, TU Delft, The Netherlands. 

He was awarded a 2005 F. Severi B.Sc. Scholarship by the Italian High-Mathematics National Institute, and a 2008 M.Sc. Fellowship by the Sant’Anna School of Advanced Studies. He was awarded 2013 and 2014 “TAC Outstanding Reviewer” by the Editorial Board of the IEEE Trans. on Automatic Control, IEEE Control Systems Society. He was recipient of the Best Paper Award at the 2016 ISDG International Conference on Network Games, Control and Optimization.
\end{IEEEbiography}

\end{document}